\documentclass[11pt]{article}

\usepackage[title]{appendix}
\usepackage{amsfonts}
\usepackage{amsmath}
\usepackage{amssymb}
\usepackage{amsthm}
\usepackage{booktabs}
\usepackage{multicol}
\usepackage{multirow}
\usepackage{tabularx}
\usepackage{diagbox}
\usepackage{graphicx}
\usepackage{subfigure}
\usepackage{array}
\usepackage{fancyhdr}
\usepackage{hyperref}
\usepackage{xcolor}
\usepackage{tikz}
\usepackage{listings}
\definecolor{mygreen}{RGB}{28,172,0} % color values Red, Green, Blue
\definecolor{mylilas}{RGB}{170,55,241}

\usepackage{matlab-prettifier}
\usepackage{color}
\usepackage[shortlabels]{enumitem}

\usepackage{algorithm}
\usepackage[noend]{algpseudocode}
\makeatletter
\def\BState{\State\hskip-\ALG@thistlm}
\makeatother
\newcommand{\tp}{\mathsf{T}}

\newcommand{\N}{\mathbb{N}}
\newcommand{\R}{\mathbb{R}}
\newcommand{\Z}{\mathbb{Z}}

\newtheorem{theorem}{Theorem}
\newtheorem{lemma}{Lemma}

\newtheorem{definition}{Definition}

\usepackage{arxiv}

\makeatletter
\def\thanks#1{\protected@xdef\@thanks{\@thanks\protect\footnotetext{#1}}}
\makeatother
\title{ON CONVERGENCE RATE OF ADAPTIVE MULTISCALE VALUE FUNCTION APPROXIMATION FOR REINFORCEMENT LEARNING\thanks{Another version of this paper has been submitted to 2019 IEEE International Workshop on 
MACHINE LEARNING FOR SIGNAL PROCESSING}}

\date{}

\author{
  Tao Li \\
  Department of Electrical and Computer Engineering\\
  New York University\\
  New York, NY 11201 \\
  \texttt{tl2636@nyu.edu} \\
   \And
 Quanyan Zhu \\
  Department of Electrical Engineering\\
  Ne York University\\
  New York, NY 11201 \\
  \texttt{qz494@nyu.edu} \\
}

\begin{document}
\maketitle

\begin{abstract}
	In this paper, we propose a generic framework for devising an adaptive approximation scheme for value function approximation in reinforcement learning, which introduces multiscale approximation. The two basic ingredients are multiresolution analysis as well as tree approximation. Starting from simple refinable functions, multiresolution analysis enables us to construct a wavelet system from which the basis functions are selected adaptively, resulting in a tree structure. Furthermore, we present the convergence rate of our multiscale approximation which does not depend on the regularity of basis functions. 	
\end{abstract}

\section{Introduction}
	In the last few decades, reinforcement learning has attracted rapidly increasing interest in machine learning communities and has made encouraging successes in solving the problem of predicting the expected long-term future cost of a stochastic dynamical system in the framework of Markov Decision Process\cite{Sutton_2018wc}. 
	
	In a reinforcement learning problem, the task of the agent, based on rewards received at each time step, is to find an optimal policy maximizing the expected long term rewards. In order to derive the optimal policy, there are mainly two approaches: value-based and policy-based methods. For value-based methods, the policy is obtained by investigating the optimal value function depicted by the Bellman equation. On the other hand, policy-based methods bypass the difficulty of analyzing the value function by focusing on optimal policy itself, namely, directly learning the optimal policy from the interactions with the environment. In this work, we focus on value-based ones, or more specifically, approximating the optimal value function by a linear combination of adaptive basis which yields a multiscale approximation.
	
	For the value function approximation, generic convergence results have been well established by \cite{tsitsiklis1997analysis},\cite{Sutton_1988bt} where the authors have illustrated that the value function can be approximated by a linear combination of basis functions. Later on, some researchers endeavored to find proper basis functions including tile coding \cite{Sutton_2018wc}, Radial Basis Functions (RBF), polynomial basis \cite{lagoudakis2003least} as well as Fourier basis \cite{konidaris2011value}. However, when using fixed basis, high performance in practice requires smart representation choices (e.g. the resolution of state space discretization in tile coding), which involves manual design and intuition. In order to automate the process of constructing suitable representations, adaptive approximation methods are of great interest. For example, adaptive tile coding (ATC) proposed in \cite{whiteson2007adaptive} starts with large tiles and gradually refines the tiles during learning by splitting existing tiles in two, yielding a state discretization with different resolutions for seperate regions. However, results on convergence analysis and approximation residual have not been known yet. 
	
    In this paper,  we move a step forward on adaptive approximation by proposing a generic framework for devising a multiscale value function approximation. The proposed framework is applicable for all basis functions that are refinable including piece-wise constant functions, piece-wise polynomials as well as other scaling functions in wavelet analysis, hence we refer to this method as generalized multiscale approximation (GMSA).  Our GMSA is a combination of multiresolution analysis \cite{meyer1995wavelets} and tree approximation \cite{COHEN2001192}. As long as the basis functions are refinable so that multiresolution analysis can be established, tree approximation performs adaptive basis selection of these refinable functions, resulting in a tree-based wavelet approximation with a multiresolution discretization of the state space and we show that ATC is just a special case of GMSA. In addition, we provided rigorous discussion about convergence rate and error analysis of our GMSA, indicating that GMSA does not rely on the regularity of basis functions.     
    	
    The rest of the paper is organized as follows. Section 2 provides some preliminaries related to MDP and multiresolution analysis. A detailed description about how tree approximation is incorporated in our GMSA is presented in Section 3, where the tree based wavelet approximation algorithm is introduced. Numerical results in Section 4 demonstrate that our method is more effective than fixed tile coding and ATC. In Section 5, we discuss the related works on adaptive representations which turn out to be a special case of our GMSA. Finally, conclusion is given in Section 6.

\section{Preliminaries}
In this section, we briefly introduce the formulation of reinforcement learning and related methodologies in dealing with this kind of sequential decision making problem. Besides, a quick review of multiresolution analysis and wavelet is also provided.
\subsection{Markov Decision Process}
A five tuple $<S,A,T,R,\gamma>$ is a Markov Decision Process (MDP) if it contains a set of states $S$ (here we consider a continuous bounded state space), a set of actions $A$, a transition dynamics function $T: S\times A\times S\to [0,1]$ and a reward function $R: S\times A\times S\to\R$. It is noted that as suggest in \cite{tsitsiklis1997analysis}, here we assume the reward function is square-integrable, i.e., $\mathbb{E}[R^2(s,a,s')]<\infty$ and hence the value function is also square-integrable. $T(s,a,s')$ denotes the probability of transitioning from state $s$ to $s'$ when taking action $a$ and $R(s,a,s')$ denotes the corresponding reward associated with this transition, while $\gamma$ is a discount factor representing how much the future reward is discounted, compared with current ones. The Markov property of MDPs lies in the fact that the dynamics in this problem is completely determined by mapping $T$, in which case the upcoming transition only depends on the current state and the action to be taken and no prior information about previous states and actions is needed. 

The ultimate goal for the agent is to find an optimal policy $\pi: S\to A, \pi=\{a_1,a_2,\cdots\}, \pi(s)=a_s\in A$, so that the discounted sum of future rewards is maximized. Notice that the rewards depend on the current state and the policy being followed, to seek the optimality, the agent need only compute the optimal value function $V^*$ which is defined as 
	$	V^*(s)=\max_{\pi}V^{\pi}(s),$
	where $V^{\pi}(s)$ is the total expected rewards starting from an initial state $s\in S$ and following a policy $\pi=\{a_1,a_2,\cdots,\}$, i.e.,
	\begin{align*}
		V^{\pi}(s)=\mathbb{E}\left\{\sum_{t=0}^{\infty} \gamma^{t} R\left(s_{t}, \pi\left(s_{t}\right), s_{t+1}\right) | s_{0}=s\right\}.
	\end{align*}
	From the description above, it is straightforward to verify that the value function satisfying the following Bellman equation:
\begin{align*}
		V^*(s)=\max_a\mathbb{E}\left\{R(s,a,s')+\gamma V^*(s')\right\}.
\end{align*}
	 When $V^*$ is known, the optimal policy $\pi^*$ is given by 
	\begin{align}\label{valuebase}
		\pi^*(s)=\arg\max_{a\in A}\mathbb{E}\left\{R(s,a,s')+\gamma V^*(s')\right\}.
	\end{align}
\subsection{Value Function Approximation}
For the value-based methods, the priority is to compute or estimate the value function and then the optimal policy can be obtained using \eqref{valuebase}. However, it is often computationally challenging to solve the equations above, especially when there are a large number of states involved in computation and as state spaces grow, the computation complexity grows exponentially. To break the curse of dimensionality, parameterized function $\tilde{V}(s,\theta)$ is constructed for approximating the optimal value function $V^*(s)$, where $\theta$ is the tuning parameter to make the approximation a good fit to the optimal value function. 

Broadly speaking, there are mainly two approximation architectures: linear and nonlinear approximations, and the difference between the two lies in the dependence of $\tilde{V}(s,\theta)$ on $\theta$. In this paper, we mainly focus on the linear approximation where the dependence is linear; namely, $\tilde{V}(s,\theta)$ can be written as a linear combination:
	$\tilde{V}(s,\theta)=\sum_{i=1}^n \theta_i b_i(s):= \theta^\tp b(s),$ where $b(s)=[b_1(s),b_2(s),\cdots,b_n(s)]^\tp$ forms a basis and $\theta=(\theta_1,\theta_2,\cdots,\theta_n)^\tp$ are the corresponding coefficients which can be determined using reinforcement learning algorithms, such as temporal difference learning TD($\lambda$) \cite{Sutton_1988bt} and a concise description is given below.

Suppose that at time $t$, the agent is at the current state $s_t$ and the parameter vector has been set to $\theta_t$, which gives an approximation value $\tilde{V}(s_t,\theta_t)$. Then, according to the transition kernel $T$, the agent is guided to the next state $s_{t+1}$ and accordingly, the temporal difference $d_t$ corresponding to the transition from $s_t$ to $s_{t+1}$ is defined as $$d_t=R(s_t,a,s_{t+1})+\gamma \tilde{V}(s_{t+1},\theta_t)-\tilde{V}(s_t,\theta_t). $$ Then, the parameter vector $\theta_t$ can be updated by the following temporal difference learning method:
$$\theta_{t+1}=\theta_t+\alpha_td_t\sum_{k=0}^t(\gamma \lambda)^{t-k}\nabla \tilde{V}(i_k,\theta_t),$$ where $\alpha_t$ is the step size and $\lambda$ is the parameter that specifies the algorithm utilized here, since temporal difference learning is actually a continuum of algorithms. One advantage of employing linear approximation, as suggested in \cite{tsitsiklis1997analysis}, is that computing the gradient is straightforward: $\nabla \tilde{V}(s,\theta)=b(s)$, which greatly simplifies the convergence proof and the training process. Throughout this paper, for the sake of simplicity, we limit our analysis to temporal difference learning; however, our adaptive basis construction is also applicable for other reinforcement learning algorithms such as SARSA and Q-learning.  

\subsection{Multiresolution Analysis}
Many applications in signal processing such as image denoising, compression and reconstruction rely on wavelet analysis or multiresolution analysis, where the given signals are approximated by its wavelet expansion, forming $L_2-$approximation with arbitrary precision. Since value function approximation in reinforcement learning is also an approximation problem, it is natural to consider leveraging this technique. 

In multiresolution analysis, a key ingredient is refinable functions, which possess the property of self-similarity. In other words, these special function can be represented by its smaller copies, or more formally defined as follows.
 \begin{definition}[Refinable function]
 	A function $\varphi\in L_2(\R)$ is a refinale function if there exists an $\ell^2$ sequence $\{h(k)\}$ such that $$\varphi(x)=2^{1/2}\sum_{k\in \Z}h(k)\varphi(2x-k).$$
 \end{definition}
On the other hand, by shifting, these $\varphi(2x-k)$ form an orthonormal system (normalization) in $L_2(\R)$ and naturally, we can construct an approximation using these shifted ones. Thus we define the following approximation operator.
 \begin{definition}
	Let $V_j=\operatorname{span}\{\varphi_{j,k},k\in \Z\}$, and the approximation operator $P_j: L_2(\R) \rightarrow V_j$ is defined as 
	\begin{align*}
		P_{j} f(x):=\sum_{k}\left\langle f, \varphi_{j, k}\right\rangle \varphi_{j, k}(x),
	\end{align*}
	where $\varphi_{j,k}(x):=2^{j/2}\varphi(2^jx-k)$, and $\langle\cdot,\cdot\rangle$ denotes the inner product in $L_2(\R)$.
\end{definition}

It is noted that if $\varphi$ is the Haar scaling function, then $P_j f$ exactly gives a piece-wise constant approximation, the simplest case in function approximation. On the other hand, from our analysis, it is straightforward that each $\varphi_{j,k}$ can be represented by the linear combination of $\varphi_{j+1,k}$(self-similarity), in other words, $V_0\subset V_1\subset\cdots\subset V_j \subset \cdots V_\infty=L_2(\R)$, thus we are able to construct a $j-$scale approximation using the basis function in $V_j$ and we can achieve an improvement in the result by moving into finer scales, i.e., considering larger $j$. However, the efficiency of constructing a finer approximation  becomes a concern.
For example, when using Haar scaling function for approximation, say we have computed the coefficient $\langle f,\varphi\rangle$, and we aim for giving a better approximation result by employing the basis in $V_1$. Unfortunately, new coefficients $\langle f, \varphi(2x)\rangle, \langle f, \varphi(2x-1)\rangle$ cannot be derived from the previous one $\langle f,\varphi\rangle$ because $V_1$ contains $V_0$ and is ``superior '' to $V_0$, hence we cannot infer $V_1$ from its subspace $V_0$. To address this drawback, we investigate the orthogonal complement of $V_j$ in $V_{j+1}.$ In order to better present our ideas in multiresolution analysis, we introduce the detail operator $Q_j$ as follows.
\begin{definition}
		Let $W_j$ be the orthogonal complement of $V_j$ in $V_{j+1}:$ $V_{j+1}=V_j\oplus W_j$. Then the detail operator $Q_j: L_2(\R)\rightarrow W_j$ is defined as $$Q_j:=P_{j+1}-P_j,$$ which projects $V_{j+1}$ onto its subspace $W_j$. 
	\end{definition}
It is noted that $Q_j$ projects the basis of $V_{j+1}$ onto $W_j$, yielding the basis of $W_j$ and actually, the introduction of $Q_j$ increases the efficiency in deriving the approximation. Considering the example above, now we can make full use of the existing information about $f$, which means $\langle f,\varphi\rangle$ is preserved and we further add some information from the orthogonal complement $W_0$ by applying $Q_0$, creating an approximation in $V_1$, since it is the direct sum of $V_0$ and $W_0$. Moreover, once we obtain the approximation in $V_1$, we can again add information from $W_1$, which gives a new approximation in $V_2$ and this process can be repeated until the error metric is below the tolerance. More mathematically, let $\psi=Q_0\varphi$ and $\psi_{j,k}(x)=2^{j/2}\psi(2^jx-k)$, then it can be easily seen that $W_j=\operatorname{span}\{\psi_{j,k},k\in \Z\}$. Hence $\{ \psi_{j,k},j\in \N,k\in \Z\}$ forms an orthonormal wavelet basis (normalization) in $L_2(\R)$ and in wavelet analysis $\psi$ is referred to as the mother wavelet. 

Finally, we conclude this part by illustrating wavelet decomposition in $L_2(\R^d)$ using the one-dimensional wavelet basis introduced above. For simplicity, in our GMSA, we only consider Haar wavelet system, whereas the framework is applicable for any compactly support wavelet functions,for example, those more complicated orthogonal wavelets in \cite{tao2019directional}. For the scaling function $\varphi$ denoted by $\psi^0$ and the mother wavelet $\psi$ denoted by $\psi^1$, we can construct $d-$dimensional wavelet basis using tensor product: let $E$ denote the collection of vertices of the unit cube in $\R^d$. For each vertex $e=(e_1,e_2,\cdots, e_d)\in E$, we define the multivariate function $\psi^e(x_1,x_2,\cdots, x_d):=\psi^{e_1}(x_1)\psi^{e_2}(x_2)\cdots\psi^{e_d}(x_d),$
and for nonzero vertex $e'$, $\psi^{e'}$ serves as one of the mother wavelets in $L_2(\R^d)$. Hence, for $j\in\N, k\in Z^d$, the wavelet function is defined as $\psi^{e'}_{j,k}(x):= 2^{jd/2}\psi^{e'}(2^jx-k),$ whose support is $I_{j,k}:=2^{-j}(k+[0,1]^d)$ and we can also denote $\psi_{j,k}^{e'}$ by $\psi_{I_{j,k}}^{e'}$. Thus, for each dyadic cube $I$, there is a collection of wavelet functions $\{\psi^{e'}_I\},e'\in E'$, where $E'$ is the set of all nonzero vertices. 

Since the domain $S$ is a bounded set, without loss of generality, we assume it is a unit cube $[0,1]^d$. Let $\mathcal{D}_j:=\cup_{k}I_{j,k}$ denote the set of all dyadic cubes with sidelength $2^{-j}$ and $\mathcal{D}=\cup_{j}\mathcal{D}_j$ denote the collection of all cubes, then the wavelet decomposition follows as
\begin{align*}
	f&= \langle f, \varphi\rangle \varphi+\sum_{j=0}^\infty\sum_{k\in Z^d}\sum_{e'\in E'}\langle f, \psi^{e'}_{j,k}\rangle\psi^{e'}_{j,k}\\
	&= \sum_{I\in \mathcal{D}_0}\sum_{e\in E}c^e_{I}(f)\psi^e_I+\sum_{j=1}^{\infty}\sum_{I\in \mathcal{D}_j}\sum_{e\in E'}c^{e'}_I(f)\psi_I^{e'}.
\end{align*}
For simplicity of notation, we introduce the following:
$$A_I(f)=\left\{
\begin{aligned}
	&\sum_{e\in E} c^e_{I}(f)\psi^e_I, I\in \mathcal{D}_0\\
	&\sum_{e\in E'} c^{e'}_I(f)\psi_I^{e'}, I\in D_j, j\geq 1.
\end{aligned}\right.
$$ Therefore, the decomposition can be rewritten as $f=\sum_{I\in D}A_I(f).$

\section{Generalized Multiscale Approximation}
As we mentioned before, even though wavelet basis makes it possible for us to devise a ``finer'' approximation by leveraging the information from the ``coarse'' one, the new approximation are based on the full basis of the new space (those from $V_j$ plus those from $W_j$), which may not be necessary all the time and could incur a huge consumption of computing resources, since the number of bases grows exponentially. In order to avoid unnecessary computation, we have to consider localization of the value function when using wavelet approximation, namely, a few coefficients with large magnitudes plus those with small magnitudes. Hence, just like the common practice in image compression, there is no need to care about basis functions with tiny coefficients and thus it is not a wise choice to add every basis function from $W_j$, instead only those with large coefficients should be paid attention to.

In fact, the same philosophy has also been presented in another approximation scheme: best $m-$term approximation in \cite{Temlyakov:1998ul}, where the author claims that taking the expansion $f=\sum_{I}c_I \psi_I$ and forming a sum of $m$ terms with largest $|c_I|$ out of this expansion gives the best approximation. However, this is not applicable in our value function approximation problem, since the coefficients are unknown at the beginning of training and it is impossible for us to rearrange them for constructing best $m-$term approximation. The difficulty in implementing best $m-$term approximation is that there is no connection between two basis functions utilized in this approximation scheme and the only way to collect these basis functions is exhaustive searching. Considering the deficiency, we now turn to tree approximation, an adaptive approximation scheme proposed in \cite{COHEN2001192} for designing a universal and progressive encoder for image compression, where all the selected basis functions are organized in a tree structure. 
\subsection{Tree Approximation}
Recall that we consider the state space $S$ as a unit cube $[0,1]^d$ and we denote by $\mathcal{D}$ all the dyadic cubes of $S$ and by $D_j$ all the cubes with side-length $2^{-j}$. For cube $I\in \mathcal{D}_j$ and cube $J\in \mathcal{D}_{j+1}$, if $J\subset I$, then we say $I$ is a parent cube of $J$ denoted by $I=\mathcal{P}(J)$ while $J$ is a child cube of $I$, denoted by $J=\mathcal{C}(I)$. For example, when $d=1$, $[0,\frac{1}{4}]$ is a child cube of $[0,\frac{1}{2}]$, whereas $[\frac{3}{4},1]$ is not. With parent-child relationship, we can construct a tree where the elements in $\mathcal{D}$ are nodes and this tree is denoted by $\mathcal{T}=\mathcal{T}(\mathcal{D})$. Before we introduce the adaptive basis selection, we first give the following definition of proper subtree.
\begin{definition}
	$\tilde{\mathcal{T}}$, a collection of nodes of $\mathcal{T}$, is a proper subtree if:  
	\begin{itemize}
	    \item the root node is $X$;
	    \item some cube $I\neq X$ is in $\tilde{\mathcal{T}}$ then its parent $\mathcal{P}(I)$ must also be in $\tilde{\mathcal{T}}$.
	\end{itemize}
\end{definition} 
With a subtree defined above, we can construct a partition $\Lambda=\Lambda(\tilde{\mathcal{T}})$ consisting of those outer leaves of a given subtree and accordingly perform the adaptive basis selection, since each dyadic cube corresponds to some wavelet functions. Therefore, basis selection is equivalent to tree construction, which is described in detail below.

Intrinsically, our tree-based wavelet approximation is constructed by thresholding its wavelet coefficients. For a given tolerance $\eta$, to generate a partition adaptively, one needs a refinement strategy: starting with the root: for the current cube $I$, one determines whether subdivide it by examining whether the corresponding coefficient $\|A_I(f)\|$ is greater than $\eta$. If $I$ is subdivided, then the same procedure will be applied recursively to its children. In the end, we obtain a proper tree $\mathcal{T}(f,\eta)$ consisting of dyadic cubes whose corresponding coefficients are greater than $\eta$ as well as their children. For each $\eta>0$,  if we define $\Lambda(f,\eta):=\{I\in \mathcal{D}_+(S):\|A_I(f)\|\geq \eta\}, \mathcal{D}_+(S)=\cup_{j\geq 1}\mathcal{D}_j$ which is a collection of all sub-cubes whose corresponding wavelet coefficients are greater than the threshold, then $\mathcal{T}(f,\eta)$ is the smallest tree containing $\Lambda(f,\eta)$. Furthermore, we associate an approximation with the proper tree $S(f,\eta)=\sum_{I\in \mathcal{T}(f,\eta)}A_I(f),$ which is referred to as tree-based wavelet approximation and is the core of our GMSA. 
\subsection{Generalized Multiscale Approximation Algorithm}
In this subsection, we introduce our GMSA in solving value function approximation problem. The basic idea is to start with an orthonormal wavelet system, of which the scaling function and the mother wavelet are defined on the state space $S$ and we keep refining the wavelet functions by scaling and shifting following a tree-based manner introduced above. To be more specific, we first initialize the basis $\phi_0$ with scaling function and mother wavelets and also parameter $\theta$ as well as eligibility trace vector. Then, parameter is updated by deploying TD($\lambda$) algorithm under current basis. Once the lowest Bellman error is achieved, we construct finer wavelet functions by applying sifting and scaling operations to those basis functions whose coefficients are greater than tolerance. Again, we roll out TD($\lambda$) with these new basis functions. This procedure is repeated until all coefficients are below the tolerance and the pseudocode is provided as follow.

\begin{algorithm}\label{algo}
	\caption{TD($\lambda$):$ S, A, T, R, \alpha_t, \gamma, p ,\phi_0, \eta$}
	\begin{algorithmic}[1]
		\State Initialization: $t\leftarrow 0, u\leftarrow 0, z_{-1}\leftarrow0, j\leftarrow 0$
		\Repeat
		\For{$i\rightarrow$ 1 to  $\operatorname{dim}(\phi_j)$}
		\State Initialize the weight vector $\theta_0$ to zero 
		\EndFor
			\Repeat 
		\State $s_{t+1} \leftarrow s_t$
		\State $\Delta V\leftarrow R(s_{t},a,s_{t+1})+\gamma V(s_{t+1},r_t)-V(s_t,r_t)$
		\State $z_t\leftarrow \gamma \lambda z_{t-1}+\phi(s_{t})$
		\State $\theta_{t+1}\leftarrow \theta_t+\alpha_t z_t\Delta V$
		\If {$|\Delta V|<$ lowest Bellman error }
		\State $u\leftarrow 0$
		\Else 
		\State $u\leftarrow u+1$
		\EndIf
		\State $t\leftarrow t+1$
			\Until{$u>p$}
		\State $\theta_j\leftarrow \theta_{t+1}$
		 \State $V\leftarrow V-\theta_j^\tp \phi_j$
		\State $j\leftarrow j+1$
		 \State perform the refinement according to value criterion: obtain new basis $\phi_j$
		\State $t\leftarrow 0$
		\Until{the norm of wavelet coefficients is below $\eta$} 
	\end{algorithmic}
\end{algorithm}
\subsection{Convergence of Generalized Multiscale Approximation}
Recall that for each $\eta>0$, we define $\Lambda(f,\eta):= \{I\in \mathcal{D}_+(S):\|A_I\|\geq \eta\}$ and throught the paper, unless specified, $\|\cdot\|$ always denotes the $L_2$ norm. Correspondingly, let $\mathcal{T}(f,\eta)$ denote the smallest tree containing $\Lambda(f,\eta)$ and we obtain the following approximation $S(f,\eta)=\sum_{I\in \mathcal{T}(f,\eta)}A_I(f).$ In order to present the convergence result and error analysis, we first characterize the smoothness or the regularity of the function using the number of cubes and this characterization is inspired DeVore's work on nonlinear approximation \cite{DeVore:1998uf}. 
\begin{definition}
	For all $\eta>0$, we define $\mathcal{B}_\lambda(L_2(S))$ as the set of functions $f\in L_2(S)$ for which there exists a constant $C(f)$ such that $\# \mathcal{T}(f,\eta)\leq C(f)\eta^{-\lambda}$, i.e., $$\mathcal{B}_\lambda(L_2(S))=\{f\in L_2(S)|\# \mathcal{T}(f,\eta)\leq C(f)\eta^{-\lambda}\},$$ 
	and the quasi-norm is defined as $|f|_{B_\lambda(L_2(S))}^\lambda=\sup_{\eta>0}\eta^\lambda \#\mathcal{T}(f,\eta).$
\end{definition}
 Then, with this definition, we finally reach the result about  convergence rate of our wavelet-based algorithm. We first introduce a lemma from Temlyakov's work on $n-$term approximation, which reveals the relation between the norm of wavelet coefficients and the number of cubes.
\begin{lemma}[Lemma 2.1 and Lemma 2.2 \cite{Temlyakov:1998ul}]
   	Let $\Lambda\subset \mathcal{D}_+$ and $S=\sum_{I\in \Lambda}A_I(S)$ then we have the following
   	\begin{align}\label{le1}
   		C\min_{I\in \Lambda}\|A_I(S)\|(\#\Lambda)^{1/2}\leq \|S\|\leq C' \max_{I\in \Lambda}\|A_I(S)\|(\#\Lambda)^{1/2},
   	\end{align}
   	where  $C,C'$ are two independent constants.
\end{lemma}
The above lemma tells that for every $f\in L_2(S)$ of the form $f=\sum_{I\in \lambda(f,\eta)}A_I(f)$, $\#\mathcal{T}(f,\eta)<C(f)\eta^{-2}$. In other words,  for $\lambda\in(0,2)$, $\mathcal{B}_{\lambda}(L_2(S))$ is nonempty. Thus, we focus our analysis on case $\lambda\in (0,2)$ and obtain the following results\footnote{In the following theorems and proofs, by $C$ we denote all involved constants, though they actually stand for different quantities.}. 
	     \begin{theorem}
	     	For $f \in \mathcal{B}_\lambda(L_2(S))$, and $\lambda\in (0,2)$, we have 
	     	\begin{align}\label{th1}
	     		\|f-S(f, \eta)\|\leq C|f|_{\mathcal{B}_{\lambda}\left(L_{2}(S)\right)}^{\lambda / 2} \eta^{1-\lambda / 2}.
	     	\end{align}
	     \end{theorem}
\begin{proof}
	For $f\in \mathcal{B}_\lambda(L_2(S))$ and let $M:=|f|_{\mathcal{B}_\lambda(L_2(S))}$. We notice that for $I\notin \mathcal{T}(f, 2^{-\ell}\eta)$, by the definition of the subtree, we have $\|A_I(f)\|\leq {2^{-\ell}\eta}.$ Further, we define the $\ell-$level-sum of $f$ as 
	$$\Sigma_{\ell} =\sum_{I \in \Delta(f,2^{-\ell})} A_I(f),$$where $\Delta(f,2^{-\ell}\eta)= \mathcal{T}\left(f, 2^{-\ell-1} \eta\right) \backslash \mathcal{T}\left(f, 2^{-\ell} \eta\right).$ Since we have $\|f-S(f,\eta)\|\leq \sum_{l=0}^\infty \|\Sigma_\ell\|$, it is sufficient for us to estimate the level sum. It is noted that for each $\Sigma_\ell$, we have 
	\begin{align*}
		\|\Sigma_\ell\|&=\|\sum_{I\in \Delta(f,2^{-\ell}\eta)}A_I(f)\|\\
		&\leq C2^{-\ell} \eta\left[\#\mathcal{T}(f, 2^{-\ell-1} \eta\right) ]^{1 / 2} \\
		&\leq C 2^{-\ell} \eta\left[M^{\lambda} 2^{\ell \lambda} \eta^{-\lambda}\right]^{1 / 2},
	\end{align*} 
	where the first inequality is due to Lemma 1 and the definition of quasi-norm leads to the second inequality. Finally, we obtain that 
	\begin{align*}
		\|f-S(f, \eta)\| & \leq \sum_{\ell=0}^{\infty}\left\|\Sigma_{\ell}\right\|\\
		&\leq C M^{\lambda / 2} \eta^{1-\lambda / 2} \sum_{\ell=0}^{\infty} 2^{-\ell(1-\lambda / 2)} \\ & \leq \tilde{C} M^{\lambda / 2} \eta^{1-\lambda / 2} \\
		&\leq \tilde{C}|f|_{\mathcal{B}_{\lambda}\left(L_{2}(S)\right)}^{\lambda / 2} \eta^{1-\lambda / 2},
	\end{align*}
	where $\tilde{C}=C\sum_{\ell=0}^{\infty} 2^{-\ell(1-\lambda / 2)}<\infty$. This completes the proof.
\end{proof}
Though the above inequality has already showed that our GMSA approximates a given function $f$ arbitrarily well as $\eta$ tends to zero, the convergence result is still unclear, since $f$ is not characterized in some well-known smoothness spaces, such as Besov spaces. Therefore, to better present the convergence, we further consider investigating approximation error for $f$ in Besov spaces. Similar wavelet expansion convergence results in Sobolev spaces have already been studied by Kon and Raphael \cite{kon2001convergence} and here we focus on Besov spaces for the following reasons. First, Besov spaces are slightly larger than Sobolev spaces and lead to more general results. Second, the smoothness in Besov spaces can be characterized by the norm of wavelet coefficients; i.e., for $f\in B_q^s(L_2(S)), 0<q<\infty$, we have the following quasi-norm:
\begin{align}\label{besov}
	|f|_{B_q^s(L_2(S))}:= \left(\sum_{j=0}^\infty 2^{jsq}\left(\sum_{I\in \mathcal{D}_j}\|A_I(f)\|^2\right)^{q/2}\right)^{1/q}.
\end{align}
For more about convergence of wavelet expansion in different smoothness spaces, we refer readers to \cite{kon2001convergence,Han_2018uc, DeVore:1998uf} and with the above definition, we are ready to bridge the gap between $\mathcal{B}_{\lambda}(L_{2}(S))$ and $B_q^s(L_2(S)).$
\begin{theorem}
	For $0<q<\infty$ and a given $\lambda\in (0,2)$, let $s=(2-\lambda)d/2\lambda$, Besov spaces $B_q^s(L_2(S))$ are continuously embedded in $\mathcal{B}_{\lambda}(L_2(S))$, i.e.,
	\begin{align}\label{emb}
		|f|_{\mathcal{B}_{\lambda}(L_{2}(S))}\leq C|f|_{B_q^s(L_2(S))}.
	\end{align}
	where $|\cdot|_{B_q^s(L_2(S))}$ denotes the quasi-norm in Besoc spaces $B_{q}^s(L_2(S)).$ If we let $N=\#\mathcal{T}(f,\eta)$, then for $f\in B_q^s(L_2(S))$ a more straightforward 
	     	\begin{align}\label{on}
	     			\|f-S(f, \eta)\| \leq C|f|_{B_q^s\left(L_{2}(S)\right)}N^{-s/d}.
	     	\end{align}
	     \end{theorem} 
\begin{proof}
    	For $f\in B_q^s(L_2(S))$, we denote $\tilde{M}$ the quasi-norm of $f$ and we also define $\Lambda_j(f,\eta):=\Lambda(f,\eta)\cap\mathcal{D}_j$ then for function $\sum_{I\in \Lambda_j(f,\eta)}A_I(f)$, the first inequality in \eqref{le1} tells that 
	$\eta (\#\Lambda_j)^{1/2}\leq C \|\sum_{I\in \Lambda_j}A_I(f)\|.$ On the other hand, by the definition of quasi-norm \eqref{besov}, we have for any $j$, {\small$$\tilde{M}^{q}=\sum_{j=0}^\infty 2^{jsq}(\sum_{I\in \mathcal{D}_j}\|A_I(f)\|^2)^{q/2}\geq 2^{jsq}(\sum_{I\in \mathcal{D}_j}\|A_I(f)\|^2)^{q/2}.$$} Combining two inequalities gives $\#\Lambda_j(f,\eta)\leq \tilde{M}^22^{-2js}\eta^{-2}.$ Furthermore, we define that $\mathcal{T}_j(f,\eta):=\mathcal{T}(f,\eta)\cap \mathcal{D}_j.$ Obviously, $\mathcal{T}_j(f,\eta)=\Lambda_j(f,\eta)\cup (\mathcal{T}_j(f,\eta)\backslash \Lambda_j(f,\eta))$ and for $I\in \mathcal{T}_j(f,\eta)\backslash \Lambda_j(f,\eta)$, its sibling must be in $\Lambda_j(f,\eta)$ and its parent must belong to $\Lambda_{j-1}(f,\eta)$. Hence, $\#\mathcal{T}_j(f,\eta)\leq \# \Lambda_j(f,\eta)+\#\Lambda_{j-1(f,\eta)}$, and we obtain
	\begin{align*}
		\# \mathcal{T}_{j}(f, \eta) &\leq C \min \left(2^{j d}, (1+2^{2s})\tilde{M}^2 2^{-2js}\eta^{-2}\right)
	\end{align*} 
	Notice that $2^{jd}$ is increasing while $(1+2^{2s})\tilde{M}^2 2^{-2js}\eta^{-2}$ is decreasing, thus there exists $j_0=\lceil \lambda/d\log_2(\sqrt{1+2^{2s}}\tilde{M}/\eta)\rceil$ such that $2^{j_0 d}\geq (1+2^{2s})\tilde{M}^2 2^{-2j_0s}\eta^{-2},$ which implies that 
	\begin{align*}
		\#\mathcal{T}(f,\eta)&=\sum_{j=0}^\infty\#\mathcal{T}_j(f,\eta)\\
		&\leq C \left(\sum_{j=0}^{j_0-1}2^{jd}+(\tilde{M}/\eta)^2\sum_{j=j_0}^\infty  2^{-2js}\right)\\
		&\leq C \tilde{M}^{\lambda}\eta^{-\lambda}.
	\end{align*}
	This inequality leads to the fact that $M\leq C\tilde{M}$, since 
	$\#\Lambda(f,\eta)\leq \#\mathcal{T}(f,\eta)\leq C\tilde{M}^\lambda \eta^{-\lambda},$ confirming \eqref{emb}. With the relation $N\sim \eta^{-\lambda}$, \eqref{on} follows from \eqref{th1} and \eqref{emb} naturally.
\end{proof}
\section{Numerical results} 
Due to the limitation of space, in this section, we describe the results of our GMSA in solving two classical control problems: Cartpole problem \cite{barto1983neuronlike} and Arcobot problem \cite{sutton1996generalization}. First, we consider Cartpole problem, where a pole is attached to a cart by an un-actuated joint. In this system, the pole starts upright and the goal is to prevent it from falling which is achieved by applying a force of $+1$(to the right) or $-1$ (to the left). The state variable is a four-tuple consisting of position and velocity of the cart as well as angle and rotation rate of the pole. Meanwhile, a reward is given for each time-step when the pole remains upright and the episode comes to an end if the pole falls or the cart moves more than 2.4 units from the center. 

To evaluate our GMSA, we test it against adaptive tile coding as well as fixed tile coding and the following parameters settings are used in our experiment: $\alpha= 0.001, \gamma=0.8, p=50, \eta=10^{-3}$. For our algorithm, Haar basis is utilized whereas 2 initial tilings for each dimension are set for ATC. Next, we test different fixed tiling representations, and select the best performing setting for comparison: for each dimension, the number of tilings is 10 and the associated width of tiles is 0.025. According to the results shown in Fig \ref{num_exp}, we can see that our GMSA achieves the best performance much better than that of fixed tile coding: our method reaches a score of 200 in fewer episodes whereas fixed tile coding needs more than $6\times 10^{6}(120\times 50000)$ episodes. Also, we notice that in this experiment, ATC is still very competitive, since our method does not outperform it significantly, even though the Haar basis outperforms the piece-wise constant basis in approximation. 

In order to illustrate further the performance of our method, we consider a more sophisticated Arcobot problem, where the value function is known to be ill-behaved \cite{munos2002variable}. The Acrobot is also 4-dimensional control problem which consists of two joints and two links, where the joint between the two links is actuated and states are depicted by angles and rotation rates of the two links. The goal is to  swing the endpoint up to a given height and there are three actions: positive torque, negative torque and no torque with a reward of -1 for each time-step. The numerical result is shown in Fig \ref{num_exp}(b), where the rolling mean is presented(rolling window is 1000 episodes). It demonstrates that our multiscale approximation learns a better representation than ATC since the mean increases faster and the rolling-out mean is always greater than that of ATC.
\begin{figure}
\label{num_exp}
	\centering
	\subfigure[Reward per episode in Cartpole problem: GMSA is compared to ATC and best-performing fixed tile coding.]{\includegraphics[width=0.46\textwidth]{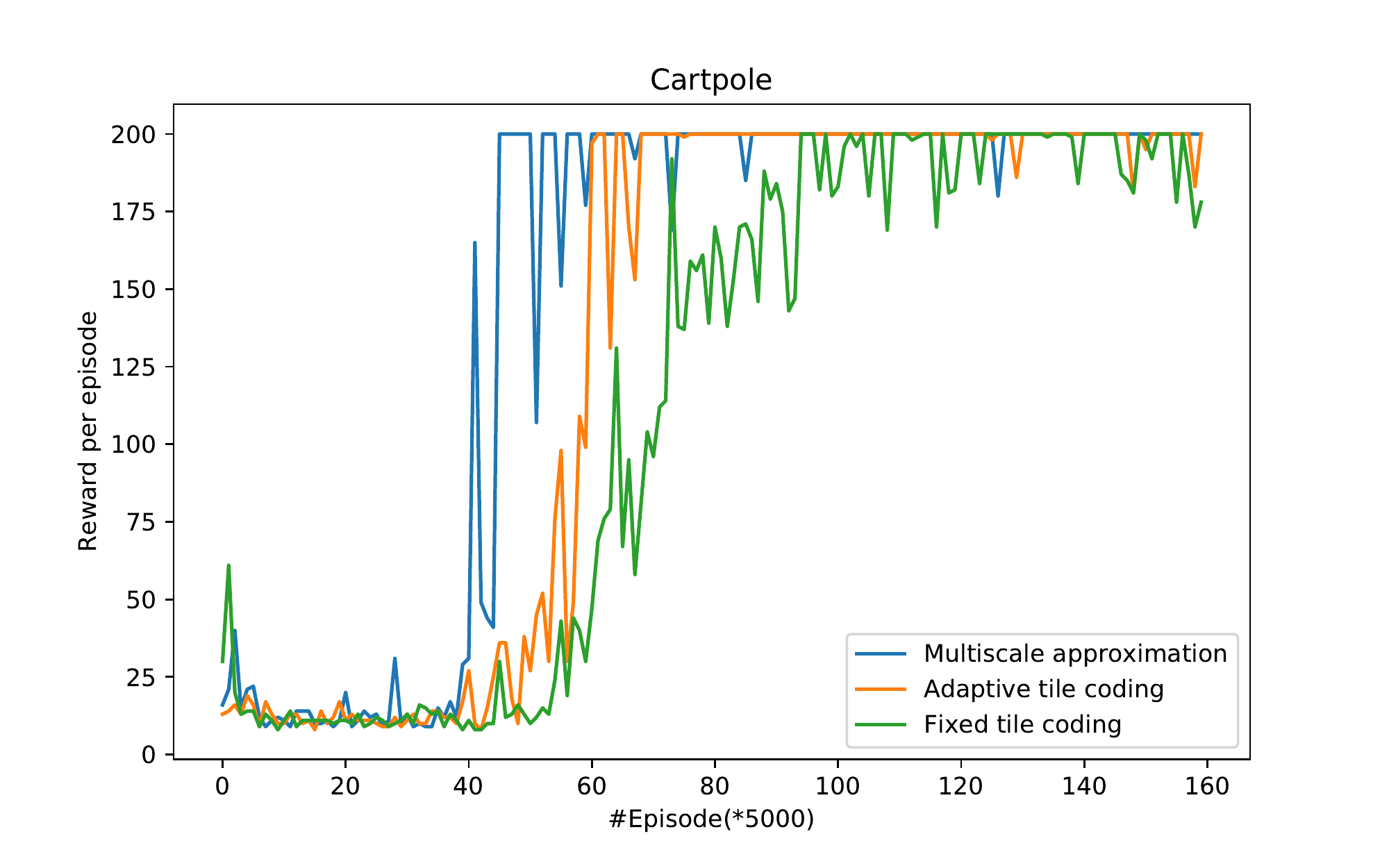}}
	\subfigure[Rolling-mean per 1000 episodes: GMSA is compared to ATC]{\includegraphics[width=0.46\textwidth]{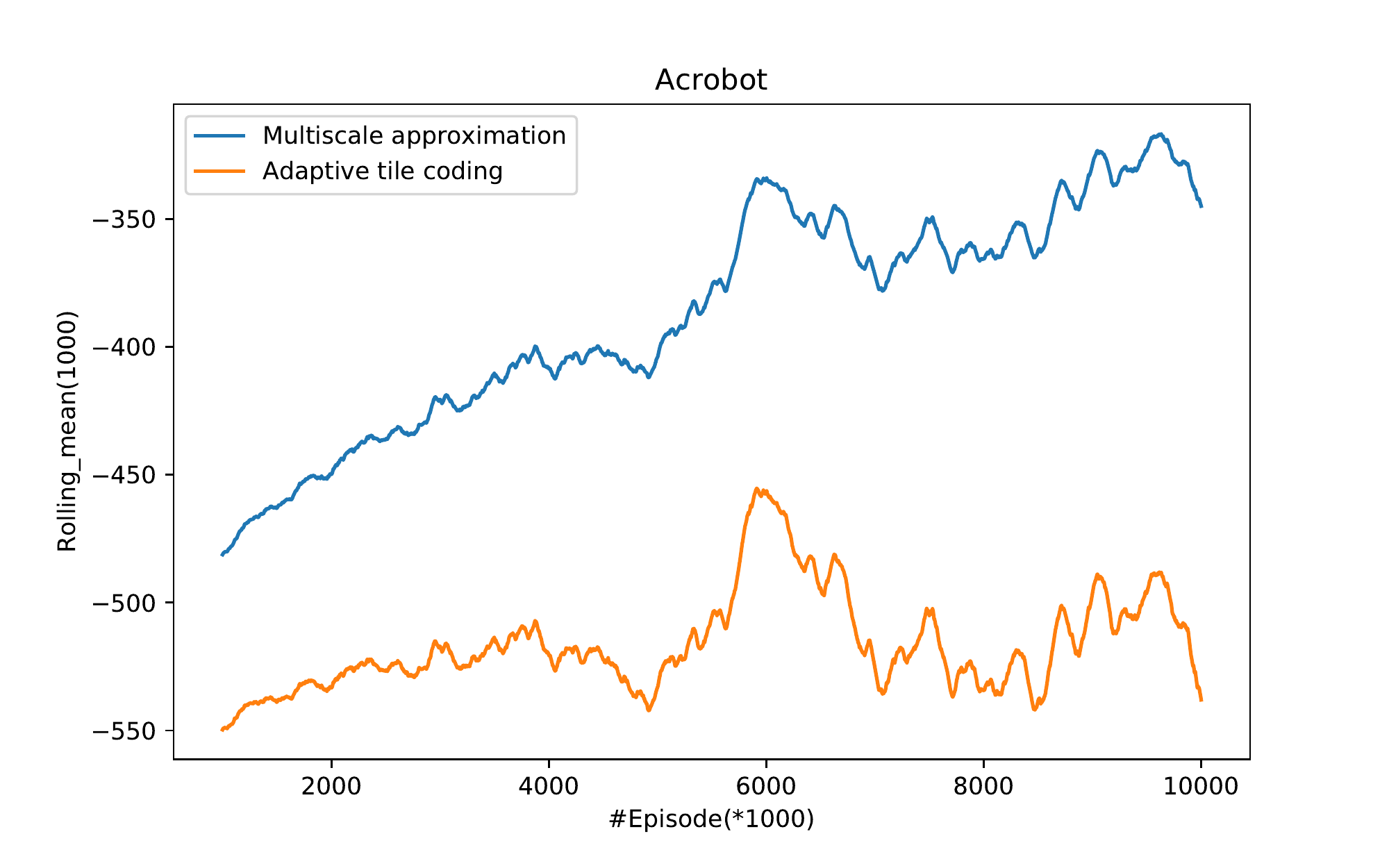}}
	\caption{Numerical results of GMSA: in comparison with ATC and fixed tile coding}
\end{figure}

\section{Discussion}
In this section, we discuss the connection between ATC and our GMSA for explaining the similarities in numerical implementation. Also, we aim to clarify that the regularity of basis functions does not affect the convergence rate. 

As we have mentioned before, our GMSA essentially constructs an adaptive representation for reinforcement learning, sharing the same idea with ATC, where piece-wise constant functions are employed for approximating the value function. Even though the tree approximations are not specified in ATC, the value criterion considered in ATC also leads to an adaptive partition which results in a tree structure. Therefore, the two methods are both based on tree approximation and the difference lies in the choice of basis. Since both follow the tree approximation, a similar argument leads to the fact that the convergence rate of using piece-wise constant function is also $O(N^{-s/d})$ and the basis functions in ATC is exactly an orthonormal system, the simplest example of multiresolution analysis, therefore, ATC is just a speical case of our GMSA. However, in numerical implementation, GMSA outperforms ATC, where only approximation operator $P_j$ is utilized but detail operator $Q_j$ never enters into the picture, This makes the previous information no longer heritable and the approximation in $V_{j}$ contribute little to the approximation in $V_{j+1}.$  Another interesting fact is that imposing regularity on the basis functions is of no use in increasing the convergence rate, since for these wavelet basis no smoothness is assumed in our proofs.   
\section{Conclusion}
In this paper, we have found that multiresolution analysis and tree approximation provide us with a generic framework for constructing adaptive approximation scheme, where the basis is created by refinable functions and associated wavelet functions. The multiresolution analysis has enabled us to leverage previous approximation and avoiding unnecessary computation. Tree approximation, on the other hand, has guaranteed that our basis selection follows an efficient way and is the key of GMSA because it can adjust approximation to the behaviors of the value function: for the regions where the function is smooth and flat, less resolution will be assigned and only the first few scales wavelet functions will be applied, whereas, for the regions where drastic changes in function value occur, more and more details shall be added. It has been noted that tree approximation does not depend on the basis, instead, it focuses on the value function itself and simply imposing regularity on basis function is infertile for achieving a higher convergence rate.
\section*{Acknowledgements}
The authors would like to thank the reviewers for their valuable feedback, especially the one who provided them with useful references and corrections. Further discussion on implementation, such as comparison with other kernel-based techniques, as suggested by reviewers, will be included in our future work. This research is supported in part by National Science Foundation (NSF) under grant ECCS-1847056, CNS-1544782, and SES-1541164, and in part by ARO grant W911NF1910041.
\bibliographystyle{ieeebib} 
\bibliography{references}

\end{document}